\RequirePackage{xcolor}
\documentclass[journal,twoside]{IEEEtran}
\usepackage{cite}
\usepackage{algorithmic}
\usepackage{algorithm}
\usepackage{graphicx}
\usepackage{booktabs} % for professional tables
\usepackage{url}
\usepackage{amsmath}
\usepackage{amssymb}
\usepackage{amsfonts}
\usepackage{mathtools}

\usepackage{enumitem}

\usepackage{amsthm}
\usepackage{enumitem}
\usepackage{tabularx}
\usepackage{multirow}
\usepackage{wrapfig}
\usepackage{pifont}
\usepackage{soul}
\usepackage[normalem]{ulem}
\usepackage[
    colorlinks=true,
    citecolor=green,
    linkcolor=black,
    urlcolor=green
]{hyperref}
\usepackage[capitalize,noabbrev]{cleveref}
\usepackage{textcomp}
\theoremstyle{plain}
\newtheorem{theorem}{Theorem}[section]
\newtheorem{proposition}[theorem]{Proposition}

\theoremstyle{definition}

\theoremstyle{remark}

\definecolor{myorange}{RGB}{255,127,0}

\newcommand{\blue}[1]{\textcolor{blue}{#1}}
\newcommand{\orange}[1]{\textcolor{myorange}{#1}}
\newcommand{\magenta}[1]{\textcolor{myorange}{#1}}

\newcommand{\figref}[1]{\hyperref[#1]{Fig.~\ref*{#1}}}

\def\BibTeX{{\rm B\kern-.05em{\sc i\kern-.025em b}\kern-.08em
    T\kern-.1667em\lower.7ex\hbox{E}\kern-.125emX}}
\begin{document}
\title{FLAT: FLow-Aligned Training of Unrolled Networks for MRI Reconstruction}
% \title{FLAT: Flow-Aligned Training of Unrolled Networks as
% Conditional Probability Flows
% in MRI Reconstruction}
\author{Kehan Qi, Saumya Gupta, Xiaoling Hu, Qingqiao Hu, Weimin Lyu, Yicun Wang, and Chao Chen
% \thanks{This paragraph of the first footnote will contain the date on which
% you submitted your paper for review. It will also contain support information,
% including sponsor and financial support acknowledgment. For example, 
% ``This work was supported in part by the U.S. Department of Commerce under Grant BS123456.'' }
\thanks{Kehan Qi and Chao Chen are with the Department of Biomedical Informatics, Stony Brook University, Stony Brook, NY 11754 USA (e-mail: kehan.qi@stonybrook.edu, chao.chen.1@stonybrook.edu).}
\thanks{Saumya Gupta, Qingqiao Hu and Weimin Lyu are with the Department of Computer Science, Stony Brook University, Stony Brook, NY 11754 USA (e-mail: saumgupta@cs.stonybrook.edu, qingqiao.hu@stonybrook.edu, welyu@cs.stonybrook.edu).}
\thanks{Xiaoling Hu is with the Massachusetts General Hospital and Harvard Medical School, Charlestown, MA 02129 USA (e-mail: xihu3@mgh.harvard.edu).}
\thanks{Yicun Wang is with the Department of Radiology, Stony Brook University, Stony Brook, NY 11754 USA (e-mail: yicun.wang@stonybrookmedicine.edu).}
\thanks{This work was partially supported by NSF Grant CCF-2144901, NIH Grants R01NS143143, R01CA297843 and R21EB040046. The content is solely the responsibility of the authors and does not necessarily represent the official views of the NIH. }
\thanks{This work has been submitted to the IEEE for possible publication. Copyright may be transferred without notice, after which this version may no longer be accessible.}
}

\maketitle

\thispagestyle{empty}
\pagestyle{empty}

\begin{abstract}
Unrolled networks are widely used in Magnetic Resonance Imaging (MRI) reconstruction for their efficiency. Structured as a series of neural network stages (or cascades), an unrolled network takes a low-quality input and passes it sequentially through each stage to iteratively refine the reconstruction.  
However, unrolled networks typically exhibit unstable output quality across cascades, resulting in sub-optimal final reconstruction results.
% propagating inaccurate predictions to the final reconstruction. 
In this work, we address this inherent limitation of unrolled networks, drawing inspiration from recent Flow Matching paradigm. We first theoretically show that unrolled networks can be viewed as discretizations of approximate conditional probability flows. This connection shows that
%provides an explicit equivalence
unrolled networks and Flow Matching are analogous in MRI reconstruction.
% Building upon this insight, we propose FLow-Aligned Training (FLAT), which (1) aligns intermediate reconstructions with the ideal Flow Matching trajectory to improve stability and convergence; and (2) derives important sub-network parameters from the Flow Matching discretization. 
Building upon this insight, we propose FLow-Aligned Training (FLAT), which (1) derives important cascade parameters from the Flow Matching discretization; and (2) aligns intermediate reconstructions with the ideal Flow Matching trajectory to improve cascade iteration stability and convergence.
Experiments on three MRI datasets show that FLAT results in a stable trajectory across sub-networks, improving the quality of the final reconstruction. 
% \cc{Take a look at the whole abstract.}
%\saum{don't you want to mention that you achieve similar performance with 3x reduction in iterations compared to diffusion models?} \kh{We do not compare diffusion models. (1) We do not try to resolve inherent issues in diffusion models. (2) The diffusion models have already employed Flow Matching to achieve such acceleration, so it's hard to claim so.}

\end{abstract}


\section{Introduction}

Reconstruction is a critical task in accelerated Magnetic Resonance Imaging (MRI), in which one takes the undersampled spatial-frequency domain (also called \textit{k}-space) data as input and recovers the corresponding high-quality image~\cite{donoho2006compressed, lustig2008compressed}. \mbox{\figref{fig:teaser}(I)} illustrates the task. This is an ill-posed problem when the undersampling rate is high, and classic methods such as compressed sensing~\cite{quan2018compressed,ye2019compressed} have been proposed to solve it. In recent years, deep learning methods have demonstrated superior performance for MRI reconstruction. In particular, unrolled networks~\cite{sun2016deep, yang2018admm, aggarwal2018modl, zhang2018ista, sriram2020end, schlemper2017deep, aghabiglou2021projection} have become one of the most successful approaches. 
%An unrolled network \cctext{takes in a low-quality image reconstructed from the undersampled frequency-domain input, and recovers the high quality reconstruction}. 
An unrolled network consists of a series of sub-networks called cascades. Each cascade 
%is a sub-network that 
takes the output of the previous cascade as input, and generates a new estimation of the image. In theory, a cascade corresponds to an unrolled iteration of a classical algorithm such as first-order optimization methods~\cite{zhang2018ista} or ADMM~\cite{sun2016deep,yang2018admm}. It is supposed to push the estimate one step closer to the final result. By solving a sequence of smaller reconstruction subproblems rather than attempting complete recovery in a single step, unrolled networks achieve promising reconstruction quality. 

\begin{figure*}
    \centerline{\includegraphics[width=\linewidth]{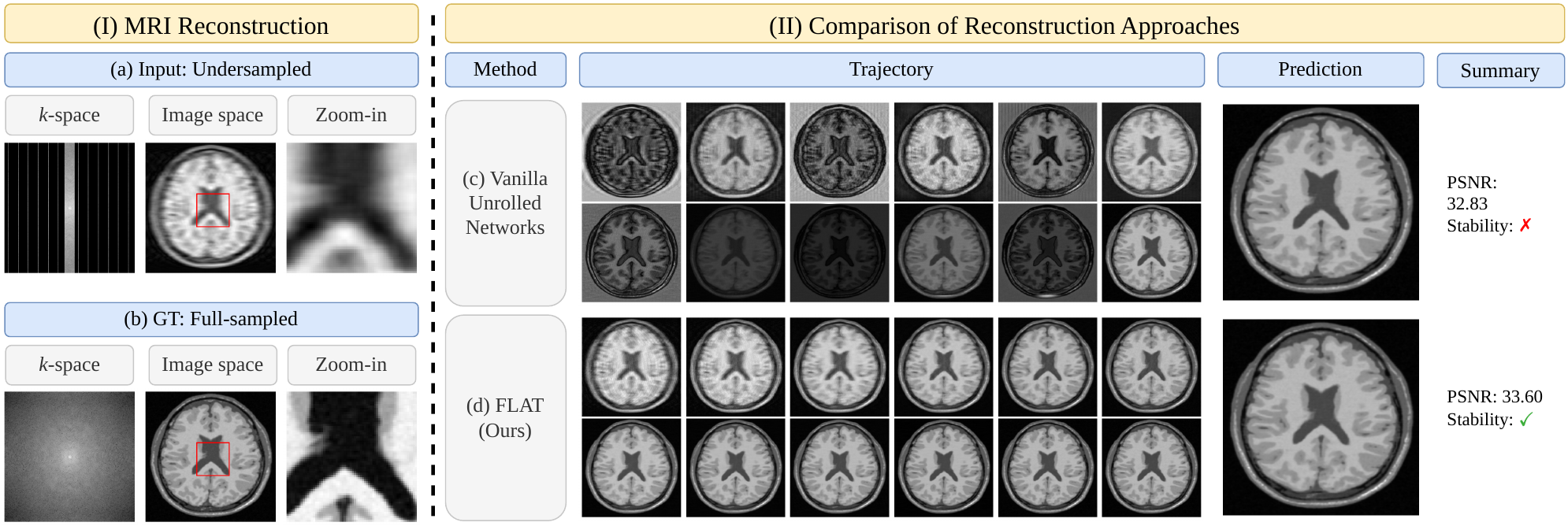}}
    \caption{
    (I) Illustrating the MRI reconstruction task: from an undersampled, aliased input (a), the task is to recover the clean, fully-sampled image (b). 
    (II) Comparison of reconstruction approaches: (c) Vanilla unrolled networks suffer from unstable image quality through the iterative cascades, impacting the final reconstruction performance.
    (d) Our FLAT, grounded in probability flow ODEs, maintains high-quality image prediction through the cascades, resulting in high reconstruction quality. 
    The `Trajectory' displays the sequential outputs of 12 cascades of the unrolled network, arranged left to right and top to bottom.
    }
    \label{fig:teaser}
\end{figure*}

% \begin{figure*}
% \vspace{-0.15in}
%     \centering
%     % \includegraphics[width=1.\linewidth]{images/teaser.pdf}
%     \centerline{\includegraphics[width=1.\linewidth]{images/icml_teaser_image-260604.pdf}}
%     \caption{
%     (I) Illustrating the MRI reconstruction task: from an undersampled, aliased input (a), the task is to recover the clean, fully-sampled image (b). 
%     (II) Comparison of reconstruction approaches: (c) Vanilla unrolled networks suffer from unstable image quality through the iterative cascades, impacting the final reconstruction performance.
%     (d) Our FLAT, grounded in probability flow ODEs, maintains high-quality image prediction through the cascades, resulting in high reconstruction quality. 
%     The `Trajectory' displays the sequential outputs of 12 cascades of the unrolled network, arranged left to right and top to bottom.
%     % The `Trajectory' indicates the outputs of 12 cascades of the unrolled networks, from left to the right, from top to the bottom.
%     % The elements in `Trajectory' arranged from left to right and top to bottom indicate the sequential outputs of 12 cascades of the unrolled networks.
%     % In the "trajectory," the elements are arranged from left to right and from top to bottom, representing the outputs of the 12 cascaded stages in the network in sequential order.
%     }
%     \label{fig:teaser}
% \end{figure*}

% \cc{New version of Paragraph 2.} 
Despite their widespread adoption, in practice, unrolled networks rarely behave as expected. As illustrated in~\mbox{\figref{fig:teaser}(II)(c)}, popular methods~\cite{zhang2018ista,sriram2020end} generate reconstructions of oscillating quality across cascades, with even later stages sometimes generating poor results. 
%their intermediate cascades often generate images with oscillatory quality; even some cascades at later stages can generate low-quality images. 
This is contrary to the aforementioned theoretical expectation that the cascades approximate an optimization and iteratively improve image quality. This also raises the question: are these unrolled networks achieving the best we can do for reconstruction? Oscillating image quality puts a lot of burden on cascades at later stages, as they use low-quality images as input but are supposed to improve significantly. 

The core issue of such uncontrollable cascade behavior stems from the lack of proper supervision for the intermediate cascades.
% Most unrolled networks are trained end-to-end and apply iteration step sizes and weight coefficients typically learned from training data~\cite{zhang2018ista,sriram2020end}. 
To supervise the reconstruction, one can compare the final cascade's output with the ground truth high-quality image \cite{chen2025comprehensive}. But this is very hard to control the training and ensure intermediate cascades learn to improve the image quality gradually. 
% With only the final cascade's output being compared with the ground truth high-quality image, it is very hard to control the training and ensure intermediate cascades learn to improve the image quality gradually. 
% One can potentially use the final target image to supervise all cascades \cite{chen2025comprehensive}. But this will not be sufficiently effective, because we are giving the intermediate cascades unrealistic goals. 
%and lacking theoretical foundation. 
Furthermore, without proper expectations for different cascades, some important parameters, e.g., the step size, which are typically optimized during training, are difficult to optimize to optimal values. 

In this paper, we address the inherent limitations of unrolled networks by drawing inspiration from the recent Flow Matching paradigm~\cite{lipman2022flow, chen2023probability, song2020score, qin2025reversing,luo2025unsupervised,luo2025upmri,zhang2024mutli}. Flow Matching characterizes data generation as a continuous probability path. By explicitly controlling intermediate timesteps along this trajectory, one can train an iterative method to follow the probability transport path.
Guided by this perspective, we revisit the image reconstruction task and demonstrate that it can be 
% approximately \cc{should remove the word "approximately".} 
formulated as an energy-based conditional probability flow that maps undersampled images to high-quality reconstructions. Within this framework, an unrolled network can be viewed essentially as a discretization of such a flow.

This connection bridges unrolled networks and Flow Matching: while unrolled networks take discrete steps, Flow Matching describes the underlying path as a continuous trajectory. This unification has two key implications: (i) intermediate supervision can be enforced by aligning the unrolled cascades with the ideal trajectory, following standard practices in Flow Matching; and (ii) critical parameters, such as step size, can be explicitly formulated using Flow Matching theory. This allows us to train a flow-aligned unrolled network with guaranteed stability, ensuring that successive cascades monotonically improve image quality toward an optimal reconstruction.

Tying it all together, we hereby propose FLow-Aligned Training (FLAT) for unrolled networks. FLAT is a training framework that constrains unrolled network parameters based on Flow Matching discretization, and provides intermediate supervision along the optimal Flow Matching path. By grounding unrolled network training in ODE theory, FLAT improves intermediate predicted image quality and enables better control over intermediate updates.
% \cc{Bad sentence, please revise.} 
As shown in~\mbox{\figref{fig:teaser} (II)(d)}, FLAT achieves high reconstruction quality, with intermediate cascades iteratively improving image quality in a stable and monotonic manner. 
%\saum{it seems you are introducing ODE in Method section instead of the beginning of the Intro. If so, we can omit writing it here.} 
% \kh{Yes, because I removed the comparison with Diffusion Models, i.e. ODE v.s. SDE.}
We evaluate FLAT on three public MRI datasets: Brainweb~\cite{cocosco1997brainweb}, MRBrainS13~\cite{mendrik2015mrbrains}, and fastMRI~\cite{zbontar2018fastmri}. Experiments show that FLAT successfully brings Flow Matching's theoretical superiority to unrolled networks and outperforms existing methods, achieving better reconstruction quality and significantly improving intermediate steps compared to vanilla unrolled networks. In summary, our contributions are as follows:
\begin{itemize}
    \item We connect Flow Matching with unrolled networks in the MRI reconstruction task by theoretically proving that unrolled networks can represent discretizations of conditional probability flows, establishing a fundamental correspondence between these two paradigms.
    \item We propose FLow-Aligned Training (FLAT) to train unrolled networks, which enforces an ODE-consistent cascade scheduling, grounds important parameters (such as step sizes and weighting terms) through the Flow Matching theory, and adds intermediate supervision to align intermediate cascades' output with the ideal Flow Matching evolving trajectory.
    % \kh{Remove interpretability here, because unrolled networks are already highly interpretable.}
    \item Extensive experiments demonstrate that FLAT improves the stability of cascading iterative process and achieves superior reconstruction results of unrolled networks.
\end{itemize}

\section{Related Work}
\label{sec:related_work}

\textbf{Deep Learning-based MRI Reconstruction.} Inspired by iterative optimization algorithms, unrolled networks such as ADMM-Net~\cite{sun2016deep, yang2018admm}, MoDL~\cite{aggarwal2018modl}, Cascaded U-Net~\cite{aghabiglou2021projection}, and E2E-VarNet~\cite{sriram2020end} unfold iterative solvers into trainable cascades that interleave learned regularization and data consistency. Transformer-based architectures~\cite{huang2022swin, guo2023reconformer, zhou2023dsformer} have been introduced to better capture long-range dependencies across image and \textit{k}-space domains. Recently, state-space models (SSMs) such as Mamba have been adapted to MRI reconstruction to combine long-range context modeling with linear-time complexity~\cite{korkmaz2025mambarecon, meng2025dh, ji2024deform, zou2025mmr, joo2025aespa}. Finally, diffusion models for accelerated MRI~\cite{xie2022measurement, cao2024high, gungor2023adaptive} established the stochastic differential equation (SDE) plus data consistency paradigm.

\textbf{Flow Matching for Image Generation.} Flow Matching links reverse SDE sampling and ODE transport~\cite{song2020score}. Earlier works~\cite{liu2022flow, albergo2022building, tong2023improving} train Continuous Normalizing Flows (CNFs) to learn maps between two data distributions. Later works such as PixelFlow~\cite{chen2025pixelflow}, HiFlow~\cite{bu2025hiflow}, STARFlow~\cite{gu2025starflow} and ResFlow~\cite{qin2025reversing} focus on image synthesis in specific domains. Additionally, \cite{yazdani2025flow} introduces Flow Matching in medical image synthesis, utilizing flow-based training for faster and higher-quality medical image generation. Recently, researchers have started to leverage the advantages of Flow Matching in MRI reconstruction~\cite{zhang2024mutli,luo2025unsupervised,luo2025upmri}. 
% \saum{If \cite{zhang2024mutli,luo2025unsupervised,luo2025upmri} are also using Flow Matching in MRI recon, we need to write atleast one line how we are different from them. Related works section is not just to list prior work, but also to show how we are different.}
% \kh{But they are Flow Matching, and ours is unrolled network, there is a lot of difference. We aim to solve problems in unrolled networks, not in Flow Matching. If we say like 'Flow Matching is not good', why don't we resolve the inherent problems in Flow Matching?}\saum{then, are we comparing against them in experiments? Usually when we list works in Related Work section, we either mention what is different/limitation in text, or, we compare against them in experiments. If we do neither, the reviewers point it out.}\kh{Yes we compared them.} Cool!
% However, as we show in~\cref{flow_matching}, vanilla Flow Matching does not work well in MRI reconstruction <mention any reason>.

\section{Method}

\begin{figure*}
% \vspace{-0.25in}
    % \centering
    \includegraphics[width=1.0\linewidth]{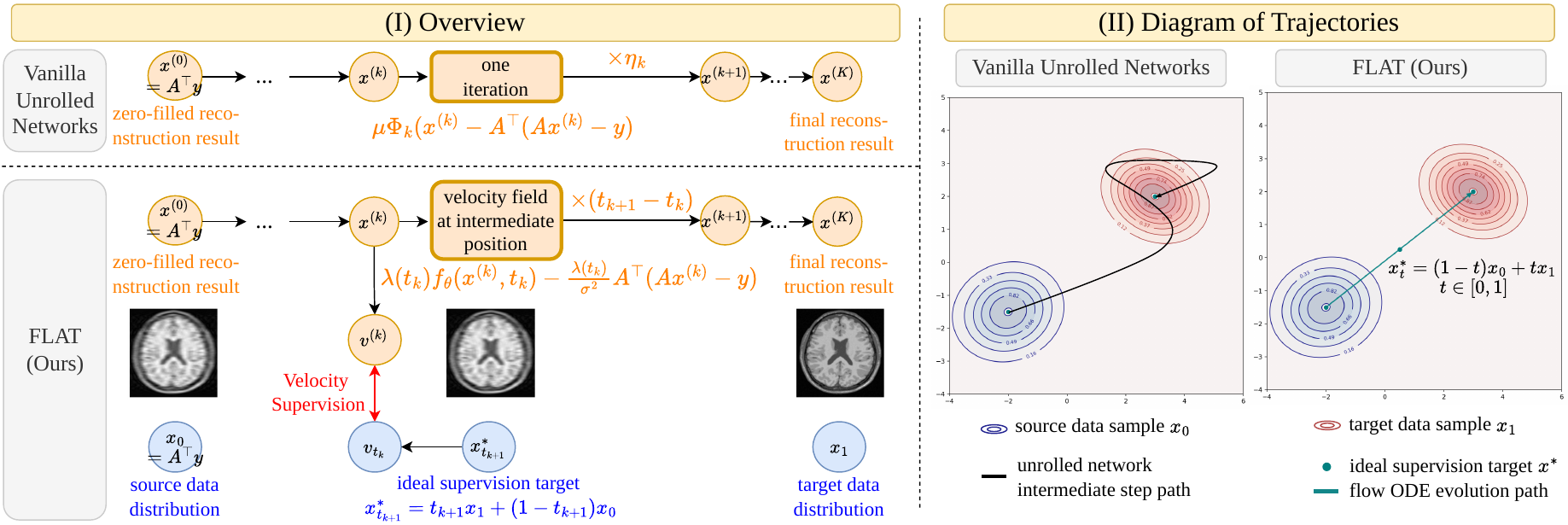}
    \caption{
    (I) Vanilla unrolled networks vs.~our FLow-Aligned Training (FLAT). Vanilla unrolled networks iteratively refine reconstructions step-by-step with supervision only at the final output. Our theory reformulates unrolling (\magenta{orange}) as a discretized flow ODE (\blue{blue}); in FLAT, each step predicts a velocity field, with intermediate supervision that aligns predicted and ideal velocities. (II) Trajectory comparison. Without intermediate supervision, vanilla unrolled networks exhibit unstable (oscillatory) image quality trajectories that ``under-run'' or ``overshoot'' the target. FLAT supervises intermediate steps to follow stable, straight-line paths guided by flow ODE theory.
    }
    \label{cflownet}
\end{figure*}

We aim to connect Flow Matching with unrolled networks in the MRI reconstruction task. We do this by establishing the theoretical correspondence between first-order gradient-based unrolled networks and energy-based conditional probability flow. 

% We first introduce the necessary background on MRI reconstruction, unrolled networks, and Flow Matching in~\cref{subsec:prelim}. Then in~\cref{ode_perspective_unrolled_net}, we present our theory that first-order gradient-based unrolled networks can be interpreted as discrete implementations of energy-based conditional probability flow, where each cascade corresponds to a discrete timestep along the continuous evolving trajectory. From this theory, we reveal that this correspondence has two critical implications: \blackcircledsmall{1} parameters such as step sizes in unrolled networks have formulations grounded in Flow Matching, and \blackcircledsmall{2} explicit supervision on intermediate steps in unrolled networks should be similar to Flow Matching. Finally, in~\cref{subsec:FLAT}, we introduce our FLow-Aligned Training (FLAT) framework, which builds upon the implications to design training strategies for unrolled networks.

\subsection{Preliminaries}
\label{subsec:prelim}
\textbf{MRI Reconstruction from Under-sampled Input.}
The standard MRI forward model can be written as
\begin{equation}
y = Ax + \epsilon,
\label{mr_sampling}
\end{equation}
where $x$ is the underlying (fully-sampled) image to be recovered, $A$ is the encoding operator which maps the desired underlying image to the undersampled k-space data, and $\epsilon \sim \mathcal{CN}(0, 2\sigma^2 I)$ i.i.d.\ complex Gaussian noise. Reconstruction is achieved by solving the inverse problem
\begin{equation}
\hat{x} = \arg \min_x \frac{1}{2} \lVert Ax - y \rVert^2 + \nu \Psi(x),
\label{mr_iterative_optimization}
\end{equation}
where $\hat{x}$ is the estimated image, $\Psi(\cdot)$ an image-domain regularization term, $\nu$ controls its strength, and $ \lVert Ax - y \rVert^2$ is the data fadility term. The adjoint operator $A^H$ maps an undersampled (zero-filled) k-space measurement back to the image domain via inverse Fourier transform and sensitivity-weighted coil combination; in particular, $A^H y$ is the zero-filled pseudo-inverse reconstruction.

\textbf{Unrolled Networks for MRI Reconstruction.} Unrolled networks are well-equipped to address the optimization problem in~\eqref{mr_iterative_optimization}. In this work, we focus on unrolled networks that utilize first-order gradients in optimization. 

% We consider E2E-VarNet as our base unrolled network, which unrolls the first-order optimization method into $K$ cascades or iterations. 

Let $x^{(0)}=A^H y$ denote the initial image reconstruction (i.e., the zero-filled reconstruction) and $x^{(K)}$ the final image after $K$ iterations. Then, the $k^{th}$ iteration is formulated as:
\begin{equation}
x^{(k+1)} = x^{(k)} -  \underbrace{\eta_kA^H  (Ax^{(k)} - y)}_{\textit{data consistency}} + \underbrace{\eta_k\mu \Phi_k(x^{(k)})}_{\textit{regularization}}
\label{unrolled_mr_recon}
\end{equation}
where $x^{(k)}$ is the image-domain estimate at iteration $k$ (with $k$ increasing from $0$ to $K-1$), $\eta_k$ is the step size, $A^H(Ax^{(k)}-y)$ is the gradient of the data-fidelity term, and $\Phi_k(\cdot)$ is a learned regularization block implemented with a convolutional neural network (CNN) acting directly on the image. This formulation is widely used in first-order gradient-based unrolled networks~\cite{zhang2018ista,sriram2020end}.
In most unrolled networks, the step size term $\eta_k$ is a learnable parameter~\cite{sun2016deep,yang2018admm,zhang2018ista,sriram2020end}, and the ideal intermediate outputs are unknown during training. In such cases, supervision is only applied to the final cascade, and no explicit supervision is applied to the intermediate outputs, leaving them under-constrained. Some studies focus on empirical analysis of step size design~\cite{chen2025comprehensive}, but are not sufficient enough because they set up unrealistic goals for intermediate steps and lack a theoretical foundation. 
% \saum{are we experimentally comparing against them?} \kh{The ablation stady, lack of I is actually deep supervision, which uses ground truth to add supervision at intermediate steps.}

\textbf{Flow Matching Based Image Generation.}
Flow Matching is a set of generative models based on Ordinary Differential Equation (ODE). They provide a continuous pathway that smoothly transitions from one distribution to another. In this work, we take Rectified Flow~\cite{liu2022flow} as the study objective, and other Flow Matching methods follow a similar pattern. Consider two image distributions $\pi_0$ (the source distribution, typically noise or low-resolution images) and $\pi_1$ (the target distribution, real or high-quality images). Rectified Flow learns a time-dependent vector field $v$ that can be used to construct a time-dependent path (called \textit{flow}) to transport samples from $\pi_0$ to $\pi_1$. Let $\{x_t \}_{t \in [0,1]}$ denote the path of a sample under this flow, defined by the ordinary differential equation $ \frac{dx_t}{dt}=v(x_t, t)$ with $x_0 \sim \pi_0$ and $x_1 \sim \pi_1$. 

In practice, $v$ is learned using a neural network $v_\theta$, and trained so that its trajectory aligns with a simple, linearly parameterized path from the source to the target. To this end, several works~\cite{liu2022flow, chen2025pixelflow, yazdani2025flow} supervise $v_\theta$ against the constant straight-line interpolation of velocity:
\begin{equation}
    x_t=tx_1+(1-t)x_0 \implies \frac{dx_t}{dt}= x_1 - x_0,
    \label{eq:ideal_trajectory}
\end{equation}
by minimizing a time-averaged least-squares objective: $\min_{\theta} \int_0^1 \mathbb{E}\left[\lVert (x_1 - x_0) - v_\theta(x_t, t) \rVert_2^2 \right] dt$. This objective encourages the learned vector field to point along the linear direction from $x_0$ toward $x_1$ at every intermediate state $x_t$.

\subsection{A Conditional Probability Flow Perspective of Unrolled Networks in MRI Reconstruction}
\label{ode_perspective_unrolled_net}
We now state and prove our main connection between first-order gradient-based unrolled networks and energy-based conditional probability flows for MRI reconstruction.

\begin{proposition}[]
\label{thm:main}
    Each cascade of a first-order gradient-based unrolled network can be viewed as one discrete step along conditional probability flow's continuous trajectory.
\end{proposition}

\begin{proof}

Let the MRI reconstruction task be modeled using a conditional ODE evolving in the image domain from the zero-filled pseudo-inverse reconstruction $x_0=A^H y$ towards the fully-sampled image $x_1$:

\begin{equation}
\frac{dx_t}{dt} = v(x_t, t;y)
\label{ode}
\end{equation}
where $v(x_t,t;y)$ denotes the velocity field at intermediate image $x_t$, timestep $t$, and conditioned on the k-space observation $y$. We first define an energy function,
\begin{equation}
    E(x;y)=\frac{1}{2\sigma^2} ||Ax-y||^2 - \log p(x)
    \label{energy_function}
\end{equation}
where $\sigma$ is a normalized scale factor, and $p(x)$ denotes an implicit image-domain prior specified through its score. This leads to the Energy Based Models (EBM)~\cite{du2019implicit}     posterior:
\begin{equation}
    p(x|y)\propto \exp(-E(x;y))
    \label{pseudo_posterior}
\end{equation}
Our conditional probability flow is then defined on this energy-based pseudo conditional probability density function.

\begin{equation}
v(x_t, t;y) = \lambda(t)\nabla_{x}\log p(x_t|y)
\label{cond_prob_ode}
\end{equation}

where $\lambda(t)$ is a time-dependent scaling factor. Decomposing \eqref{cond_prob_ode} with \eqref{energy_function} and \eqref{pseudo_posterior} gives
\begin{equation}
v(x_t,t;y) = \lambda(t)\nabla_x \log p(x_t) - \lambda(t) A^H (Ax_t-y)/\sigma^2
\end{equation}
Following~\cite{song2020score}, $\nabla_x\log p(x_t)$ is approximated by a neural network $f_\theta(x_t,t)$ with parameter $\theta$:
\begin{equation}
v_\theta(x_t, t;y) = \lambda(t) f_\theta(x_t, t) - \lambda(t) A^H (Ax_t - y) / \sigma^2
\label{neural_network_ode}
\end{equation}
To solve this continuous ODE, we apply numerical discretization. Using a forward Euler step from $t_k$ to $t_{k+1}$ with $\delta_k = t_{k+1} - t_{k}$:
\begin{equation}
\begin{aligned}
x_{t_{k+1}} = x_{t_{k}} &+  
\underbrace{
\delta_k \lambda(t_k) f_\theta(x_{t_k}, t_k)
}_{\textit{pseudo }prior} \\
&- 
\underbrace{
\delta_k \lambda(t_k) A^H (Ax_{t_k} - y) / \sigma^2
}_{\textit{data consistency}}
\label{discretized_ode}
\end{aligned}
\end{equation}

For clarity, we color-code terms from \blue{discretized conditional probability flow} and \magenta{unrolled network iteration}. 
Notice how~\blue{\eqref{discretized_ode}}'s 
data consistency and pseudo prior
terms respectively correspond to the data consistency and regularization terms from~\magenta{\eqref{unrolled_mr_recon}}. We thus establish the following correspondence:
% Notice how~\eqref{discretized_ode}'s 
% data consistency and pseudo prior
% terms respectively correspond to the data consistency and regularization terms from~\eqref{unrolled_mr_recon}. We thus establish the following correspondence, where in each equation the left term indicates the unrolled networks, and the right term indicates discretized conditional probability flow:
\begin{equation}
\begin{aligned}
\orange{x^{(k)}} = \blue{x_{t_k}} ,
\hspace{0.5cm}
\orange{\eta_k} &= \blue{\delta_k \lambda(t_k) / \sigma^2}, \\
% \hspace{0.5cm}
\orange{\mu} = \blue{\sigma^2} ,
\hspace{0.5cm}
\orange{\Phi_k(x_k)} &= \blue{f_\theta(x_{t_k}, t_k)}, \\
\orange{-\eta_kA^H  (Ax^{(k)} - y) + \eta_k\mu \Phi_k(x^{(k)})} &= \blue{v_\theta(x_t, t;y)}
\label{eq:correspondence}
\end{aligned}
\end{equation}

% \begin{align}
% \magenta{x^{(k)}} = \blue{x_{t_k}} ,
% \hspace{0.3cm}
% \magenta{\eta_k} &= \blue{\delta_k \lambda(t_k) / \sigma^2},
% \hspace{0.3cm}
% \magenta{\mu} = \blue{\sigma^2} , \nonumber
% \\
% \magenta{\Phi_k(x_k)} = \blue{f_\theta(x_{t_k}, t_k)}, \nonumber \\
% \magenta{-\eta_kA^H  (Ax^{(k)} - y) + \eta_k\mu \Phi_k(x^{(k)})} &= \blue{v_\theta(x_t, t;y)}
% \label{eq:correspondence}
% \end{align}

% \begin{align}
% x^{(k)} &= x_{t_k} ,
% % \hspace{0.3cm}
% \eta_k &= \delta_k \lambda(t_k) / \sigma^2, \\
% % \hspace{0.3cm}
% \mu &= \sigma^2 , 
% \Phi_k(x_k) &= f_\theta(x_{t_k}, t_k), \nonumber \\
% -\eta_kA^H  (Ax^{(k)} - y) + \eta_k\mu \Phi_k(x^{(k)}) &= v_\theta(x_t, t;y)
% \label{eq:correspondence}
% \end{align}

\end{proof}

The mapping from~\eqref{eq:correspondence} reveals the implicit connection between unrolled networks and Flow Matching. The sequence of reconstructions $\{x^{(1)}, \dots, x^{(K)}\}$ in unrolled networks 
% \saum{should it be x1 instead of x0?} 
forms a discretized trajectory approximating the continuous Flow Matching solution. This unifies the first-order gradient-based unrolled network and energy-based Flow Matching to a single theoretical framework.

This new perspective provides a continuous-time theoretical foundation for a previously discrete and empirically-driven class of models. 
% and an end-to-end modeling empirically practice for a previously intermediate-step-driven model training, 
% opening up new avenues for principled model design and training. \saum{The previous sentence is too long. Also, it is not clear which method was previously intermediate-step-driven? New sentence sample: 
% This new perspective provides a continuous-time theoretical foundation for a previously discrete, empirically-driven class of models, and enables end-to-end training for methods that previously relied on intermediate steps. Together, these contributions open new avenues for principled model design and training.}
% \kh{I don't remember what intermediate-step-driven model is.}
From the mapping, we have the following key implications: 
\begin{enumerate}[label=\protect{\Roman*}]
    \item \textbf{Parameters in unrolled networks are grounded in Flow Matching:} Timestep $\delta_k$, step-size $\eta_k$ and weight $\mu$ are not free. This is because, from the mapping, we obtain $\sum_k\delta_k=1$, $\eta_k = \delta_k \lambda(t_k) / \sigma^2$ and $\mu = \sigma^2$. This is important: if $\eta_k$ is left free (as in methods ADMM-Net, E2E-VarNet, ISTA-Net, Cascaded U-Net, etc.), then the effective timesteps can zigzag or collapse, resulting in erratic intermediate images. 
    %\saum{some reviewers mentioned that sum delta = 1 is wrong. Do you want to add more details here or in sec 3.3? (because of setting $\sigma=1$)} \kh{This is similar to the coefficient term in score-based models. In the original DDPM paper, the author also found that the final result is insensitive to this term, as long as they set it as a constant. Later works show that there are some techniques to improve this $\sigma$ value by predicting an interpolation factor for it, but that does not fit our theoretical framework. Besides, no matter what $\sum \delta_t$ we use, the $x_t$ should iterate from $x_0$ to $x_K$. If we fix $x_K$ as $x_1$ just like what Flow Matching is doing, then there should be $\sum \delta_t=1$. This reviewer's concern also works for all Flow Matching approaches: why evolve from $x_0$ to $x_1$?}\saum{okay cool.}
    \item \textbf{Intermediate supervision in unrolled networks governed by Flow Matching:} We can introduce Flow Matching training paradigms for unrolled networks. Since $x^{(k)} = x_{t_k}$, supervising at intermediate steps using the appropriate Flow-Matching-consistent training targets (i.e., ideal evolving velocity at intermediate steps) aligns the unrolled network with the ideal Flow Matching evolving trajectory, thereby improving stability and convergence.
\end{enumerate}

\begin{figure*}
    \centering
    \includegraphics[width=1.\linewidth]{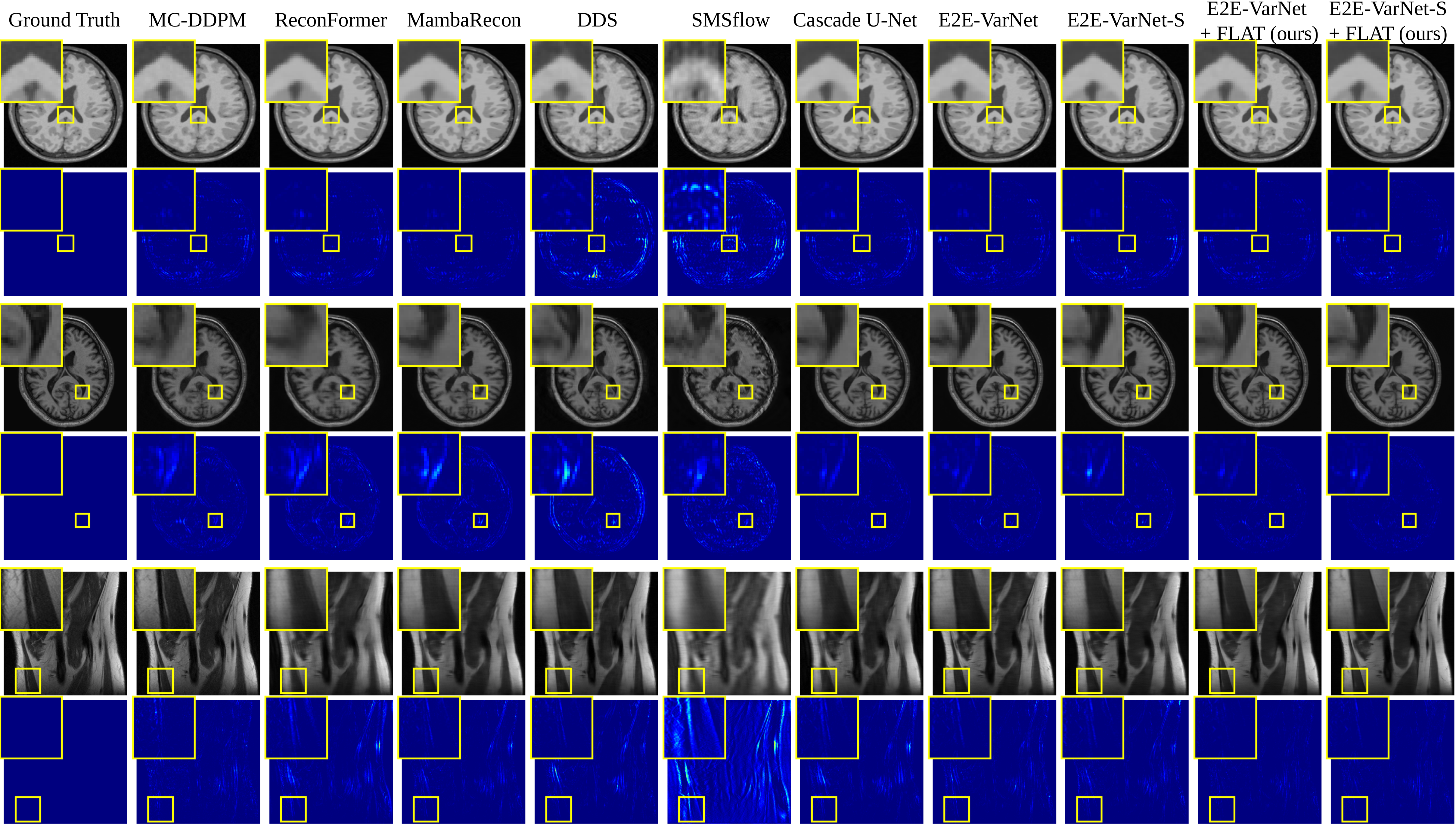}
    \caption{Qualitative results on Brainweb (rows 1-2), MRBrainS13 (rows 3-4), and fastMRI knee (rows 5-6). For each dataset, the first row shows reconstructions; the second row shows the squared-error map relative to the ground truth to visualize the error magnitude. As shown in the error map, our FLAT (the two columns to the right) achieves the best results. Yellow inset shows zoomed-in regions.
    }
    \label{fig:qualitative_results}
\end{figure*}

\subsection{FLow-Aligned Training Strategy (FLAT)}
\label{subsec:FLAT}
From the implications \uppercase\expandafter{\romannumeral 1} and \uppercase\expandafter{\romannumeral 2}, we propose FLow-Aligned Training (FLAT), which brings the Flow Matching training paradigm to unrolled networks in MRI reconstruction. FLAT \textit{(i)} enforces a time schedule by choosing a monotone sequence $\{t_k \}_{k=0}^{K}$ with $\delta_k = t_{k+1} - t_k$ and $\sum_k\delta_k=1$ and sets parameters $\eta_k$ and $\mu$ to conform to Flow Matching restrictions, and \textit{(ii)} adds intermediate supervision that aligns each cascade with the ideal flow trajectory. \figref{cflownet} illustrates the contrast between vanilla unrolled networks and FLAT.

% \saum{are we keeping the time schedule and sum delta = 1? if yes, then bring it back into the `implications' paragraphs above} \kh{This has already been in implication 1. I merged implication 1 and 2, because in ablation I put them together.}

\textbf{Scheduling and Parameters.} Following Flow Matching, we first fix the time schedule $\{ t_k \}$ to cover the full horizon $[0 \to 1]$. This induces a coherent, monotone cascade schedule. Given $\{ t_k \}$, we compute the parameters $\delta_k$, $\eta_k$ and $\mu$  directly from~\eqref{eq:correspondence}.

\textbf{Velocity Alignment.} Similar to the training process of Flow Matching~\cite{liu2022flow}, we supervise the `velocity' at intermediate steps. We define the ideal discretized velocity at the $k$-th timestep, i.e. the $k$-th cascade of the unrolled network, as the discrete temporal derivative: $v_{t_k} = (x^*_{t_{k+1}} - x^{(k)}) / (t_{k+1} - t_k)$ where $x^*_{t_{k+1}}$ represents the ground truth linearly interpolated at time $t_{k+1}$, and $x^{(k)}$ is the network prediction at the $k$-th iteration. This velocity is computed using the previous iterative result as its start point. Similarly, the network's predicted velocity is: $v^{(k)} = (x^{(k+1)} - x^{(k)})/(t_{k+1} - t_k)$ where $x^{(k+1)}$ denotes the network's output at step $k+1$. The velocity alignment loss at timestep $t$ is 
formulated as: 
\begin{align}
\mathcal{L}_{\text{velocity}}(t_k) =  \lvert v_{t_k} - v^{(k)} \rvert 
\end{align}
This velocity supervision provides a strong inductive bias that guides the network to learn physically meaningful transitions between consecutive states, leading to more stable and accurate reconstruction flows. 

\textbf{Training Objective.} The complete training objective combines velocity supervision $\mathcal{L}_{\text{velocity}}$ with the standard reconstruction loss to ensure both dynamic correctness and reconstruction quality:
\begin{equation}
\mathcal{L}_{\text{FLAT}} = \mathcal{L}_{\text{reconstruction}} + w_\text{velocity} \sum_{k=0}^{K-1}\mathcal{L}_{\text{velocity}}(t_k)
\end{equation}

$L_\text{reconstruction}$ is the standard loss in the MRI reconstruction task. Our $L_\text{velocity}$ is an additional objective to stabilize the iteration process to keep it closer to the ideal one.

% \saum{do you want to mention that it is MAE in our experiments?} \kh{Finally I used SSIM loss, not MAE.}
 
\section{Experiments}
\label{sec:experiments}

\begin{table*}
\centering
\caption{Comparison with existing MRI reconstruction approaches. Our method FLAT achieves the statistically significant best PSNR on all three datasets, and numerically best SSIM on two datasets. The statistically significantly best performance is highlighted with \textbf{bold}, and the top-3 numerically best performances are highlighted with \underline{underline}. 
%Our FLAT achieves either statistically significant best performance or one of the best metrics among all approaches.
}
% \kh{Most of the papers do not mark the `method type', like if they are diffusion models, VLMs, Flow Matching, GANs, etc. Do we need to do that?}
\label{tab:comparison}
\scriptsize
\setlength{\tabcolsep}{3pt}
\centering
\begin{tabular}{c|cc|cc|cc}
\hline
\multirow{2}{*}{Method}                                               & \multicolumn{2}{c|}{Brainweb}                                          & \multicolumn{2}{c|}{MRBrainS13}                               & \multicolumn{2}{c}{fastMRI Knee}                                       \\ \cline{2-7} 
                                                                      & PSNR $\uparrow$                   & SSIM $\uparrow$                    & PSNR $\uparrow$                   & SSIM $\uparrow$           & PSNR $\uparrow$                   & SSIM $\uparrow$                    \\ \hline
MC-DDPM                                                               & 32.52 ± 3.9800                    & 0.9194 ± 0.0348                    & 28.15 ± 2.3492                    & 0.8588 ± 0.0554           & 29.39 ± 3.5800                    & 0.5785 ± 0.1713                    \\
ReconFormer                                                           & 30.75 ± 3.3964                    & 0.8984 ± 0.0587                    & 30.22 ± 2.5339                    & 0.8683 ± 0.0503           & 30.55 ± 3.4174                    & 0.6691 ± 0.1398                    \\
MambaRecon                                                            & \underline{33.25 ± 3.2946}          & 0.9219 ± 0.0764                    & 28.62 ± 2.3664                    & 0.8816 ± 0.0541           & 27.22 ± 3.0227                    & 0.5474 ± 0.1647                    \\
DDS                                                                   & 28.00 ± 3.4112                    & 0.8337 ± 0.0527                    & 28.15 ± 2.4105                    & 0.8242 ± 0.0603           & 30.08 ± 3.2659                    & 0.6126 ± 0.1374                    \\
SMSflow                                                               & 27.56 ± 3.3404                    & 0.8037 ± 0.0786                    & 27.20 ± 1.7209                    & 0.8158 ± 0.0593           & 22.38 ± 0.8165                    & 0.5032 ± 0.0751                    \\
Cascaded U-Net                                                        & 31.80 ± 2.9219                    & 0.9119 ± 0.0292                    & 29.85 ± 2.5413                    & 0.9009 ± 0.0484           & 31.01 ± 3.3856                    & 0.6749 ± 0.1390                    \\
E2E-VarNet                                                            & 32.77 ± 3.5402                    & \underline{0.9320 ± 0.0314}          & \underline{32.59 ± 2.9929}          & \underline{0.9226 ± 0.0423} & \underline{31.18 ± 3.4983}          & \textbf{\underline{0.7036 ± 0.1344}} \\
E2E-VarNet-S                                                          & 32.81 ± 3.4072                    & 0.9295 ± 0.0304                    & 32.44 ± 3.0097                    & 0.9188 ± 0.0439           & 31.14 ± 3.4538                    & \underline{0.6915 ± 0.1372}          \\ \hline
\begin{tabular}[c]{@{}c@{}}E2E-VarNet \\ + FLAT (Ours)\end{tabular}   & \textbf{\underline{33.62 ± 3.3752}} & \textbf{\underline{0.9412 ± 0.0269}} & \textbf{\underline{33.44 ± 3.0523}} & \underline{0.9256 ± 0.0427} & \textbf{\underline{31.54 ± 3.6451}} & \underline{0.6868 ± 0.1479}          \\
\begin{tabular}[c]{@{}c@{}}E2E-VarNet-S \\ + FLAT (Ours)\end{tabular} & \underline{33.09 ± 3.2174}          & \underline{0.9305 ± 0.0276}          & \textbf{\underline{33.23 ± 3.0539}} & \underline{0.9235 ± 0.0430} & \textbf{\underline{31.51 ± 3.6389}} & 0.6835 ± 0.1474                    \\ \hline
\end{tabular}
\end{table*}

% \subsection{Datasets and Baselines}

\textbf{Datasets.} We evaluate on three public MRI datasets: Brainweb~\cite{cocosco1997brainweb}, MRBrainS13~\cite{mendrik2015mrbrains}, and fastMRI knee dataset~\cite{zbontar2018fastmri}. For all datasets, we employ 1-D equispaced fraction sampling on 2-D slices with $8\times$ acceleration and center fraction $8\%$ to simulate undersampling. We also analyzed random sampling in the ablation study. Dataset details are in~\cref{appendix:dataset_description}.

\textbf{Baselines.} We compare against following baselines: Cascaded U-Net~\cite{aghabiglou2021projection}, E2E-VarNet~\cite{sriram2020end}, E2E-VarNet-S~\cite{sriram2020end}, ReconFormer~\cite{guo2023reconformer}, MambaRecon~\cite{korkmaz2025mambarecon}, MC-DDPM~\cite{xie2022measurement}, DDS~\cite{chung2023decomposed}, and SMSflow~\cite{zhang2024mutli}. For each baseline, we use the loss terms employed in the original paper to achieve the best performance. 
% \saum{Kehan, i split the baselines into categories. Just check if its correct}
% \kh{Typically correct, but I would classify them with other labels. Transformer, Mamba, unrolled networks, Diffusion Models, and Flow Matching.}

\textbf{Implementation.} Though our FLAT is backbone-agnostic, we use the widely-used MRI reconstruction network E2E-VarNet~\cite{sriram2020end} as our base model. Implementation details are described in~\cref{appendix:implementation_details}.

\textbf{Evaluation Metrics.} We evaluate using Peak Signal-to-Noise Ratio (PSNR)~\cite{hore2010image} and Structural Similarity Index (SSIM)~\cite{wang2004image} which are standard metrics for MRI reconstruction task. We also perform the unpaired t-test~\cite{student1908probable} (95\% confidence interval) to determine statistical significance of improvement. 

\subsection{Results}

\cref{tab:comparison} reports quantitative results across all datasets.
% (due to space constraints, we provide results on the fastMRI multi-coil knee dataset in~\cref{impact_of_flat_on_fastmri_multi_coil}).
% \kh{Multi-coil dataset experiments should not be `due to space constraints blabla' because it is not conducted on all baselines. It is more like an ablation study. Therefore, I move it back.}
On Brainweb and MRBrainS13 datasets, FLAT achieves either the statistically significant best performance (\textbf{bold}) or one of the numerically top-3 performance (\underline{underline}). This superior performance stems from FLAT’s ODE-consistent update schedule and intermediate supervision, making each cascade contribute meaningfully to the final reconstruction. 
\begin{figure*}
    \centering
    \centerline{\includegraphics[width=1.\linewidth]{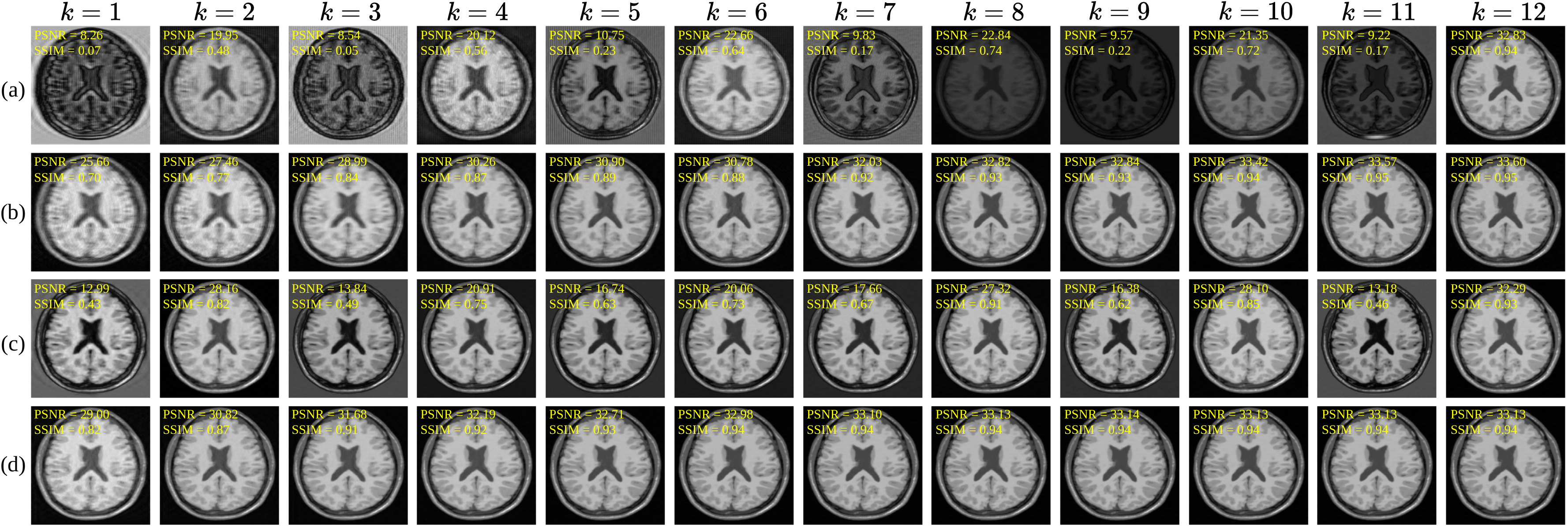}}
    \caption{
    Visualizing intermediate steps of (a) E2E-VarNet (b) E2E-VarNet + FLAT (c) E2E-VarNet-S (d) E2E-VarNet-S + FLAT. We observe an unstable trajectory in (a)(c), with half of the steps suffering from image quality degradation. In (b)(d), FLAT forces the network to approximate the ideal intermediate results similar to Flow Matching, leading to better image quality and final results.
    }
    \label{fig:intermediate}
\end{figure*}

%    Visualization of intermediate steps of unrolled networks and our FLAT. (a,c): Intermediate steps of E2E-VarNet and E2E-VarNet-S. The image quality visualization and metrics perform instability, with half of the steps suffering from image quality degradation. (b, d): Intermediate steps of E2E-VarNet + FLAT and E2E-VarNet-S + FLAT. We force the network to approximate the ideal intermediate status similar to Flow Matching, leading to increasing image quality and better final results.

Qualitative results in~\figref{fig:qualitative_results} mirror the quantitative findings. FLAT achieves high image quality by reducing noise and artifacts, and recovering fine anatomical details compared to baselines. The squared-error maps show low-magnitude errors for FLAT vs.~higher 
%-magnitude
errors for baselines. 

We visualize intermediate steps in~\figref{fig:intermediate}. Vanilla E2E-VarNet has unstable cascades outputs, while FLAT maintains stable ones. 
% Visualization of intermediate steps in~\figref{fig:intermediate} illustrates the unstable trajectory of vanilla E2E-VarNet and the maintenance of information in our FLAT. 
In E2E-VarNet, outputs of some cascades ($k=1,3,5,...,11$) suffer from low image quality, causing the following cascades ($k=2,4,6,...,12$) to make less accurate predictions. E2E-VarNet-S has better iterative predictions across cascades, but still suffers from this issue. In contrast, our FLAT benefits from intermediate supervision and theoretically grounded parameter-setting, leading to a stable and increasingly better image quality across cascades. Thus, FLAT produces a higher quality final reconstruction.

\textbf{Takeaway. }These results indicate that FLAT successfully overcomes low-quality intermediate steps typical of unrolled networks. Leveraging progressively improved input images, it makes more accurate predictions at each stage. Ultimately, FLAT is able to deliver higher quality reconstructions through ODE-aligned parameter-setting and intermediate supervision. 
% \kh{I am not sure what this takeaway wants to inform the reader.}\saum{i did not include this Takeaway. It was already there, but i fixed the english}
%and leverages improving input images to make more accurate predictions, delivering higher quality through ODE-aligned updates and intermediate supervision. 
% \saum{i think we should remove boundary/background here as it is not mentioned before in the paper}

\subsection{Ablation Studies}
To demonstrate the efficacy of FLAT, we conduct comprehensive ablation studies on the Brainweb dataset. We analyze individual contributions of components in~\cref{impact_of_components}, formulation of $L_\text{velocity}$ in~\cref{impact_of_formulation_of_L_velocity}, the impact of MRI undersampling level in~\cref{impact_of_mri_undersampling_level}, the trajectory stability in~\cref{cascade_stability}, the impact of $w_\text{velocity}$, $\alpha$, $K$ and $\sigma$ in~\cref{impact_of_alpha_and_k_and_sigma}, the impact of FLAT on multi-slice reconstruction in~\cref{impact_of_joint_multi_slice_brainweb}, the impact of FLAT on multi-coil data in~\cref{impact_of_flat_on_fastmri_multi_coil}, the impact of FLAT on ADMM-CSNet in~\cref{sec:impact_of_flat_on_admm_csnet}, and a mutual mechanism on Flow Matching in~\cref{flow_matching}. 
% Due to space constraints, we discuss impact of $w_\text{velocity}$, $\alpha$, $K$ and $\sigma$ in~\cref{impact_of_alpha_and_k_and_sigma}. 
% We also explore the impact of FLAT on multi-coil data in~\cref{impact_of_flat_on_fastmri_multi_coil}, and a mutual mechanism on Flow Matching in~\cref{flow_matching}.
% \saum{this line is not clear. Mechanism similar to what ?}
% \kh{Mutual. Explained in the Appendix.}

\subsubsection{Impact of Components}
\label{impact_of_components}

We evaluate two key components of FLAT: (I) explicitly setting the ODE-derived parameters $\{ \delta_k \}_{k=1}^K$, $\{\eta_k\}$, and $\mu$, and (II) using intermediate (velocity) supervision. 
From~\cref{tab:ablation_component_flat}, it is clear that the combination of these two components yields the largest gain for both unshared and shared weights of unrolled networks, as it fully leverages the Flow Matching perspective. 
%, leading to 0.94dB and 0.45dB in PSNR, respectively.
%From~\cref{tab:ablation_component_flat}, it is clear that the combination of these two components yields the largest gain by fully leveraging the Flow Matching perspective for both unshared and shared weights of unrolled networks, leading to 0.94dB and 0.45dB in PSNR, respectively.

\begin{table}
\centering
\caption{Ablation of (I) (II) components in FLAT. We \textbf{bold} the best combination, which is also statistically significantly better than the base model. Each of these two components individually improves the reconstruction performance, and combining them yields the best effect, demonstrating that they interact with each other to provide even greater benefits.
% \saum{you bolded the wrong value for SSIM. 0.9312 is better than 0.9305}
% \kh{Oops}
}
\label{tab:ablation_component_flat}
\scriptsize
\setlength{\tabcolsep}{8pt}
\begin{tabular}{c|cc|cc}
\hline
Base Model                      & I & II & PSNR $\uparrow$                    & SSIM $\uparrow$                     \\ \hline
\multirow{4}{*}{E2E-VarNet}     & \ding{55}             & \ding{55}             & 32.68 ± 3.3535          & 0.9324 ± 0.0291          \\
                                & \checkmark          & \ding{55}             & 33.06 ± 3.3442          & 0.9361 ± 0.0282          \\
                                & \ding{55}             & \checkmark          & 33.39 ± 3.7235          & 0.9406 ± 0.0291          \\
                                & \checkmark          & \checkmark          & \textbf{33.62 ± 3.3752} & \textbf{0.9412 ± 0.0269} \\ \hline
\multirow{4}{*}{E2E-VarNet-S} & \ding{55}             & \ding{55}             & 32.64 ± 3.2147          & 0.9283 ± 0.0279          \\
                                & \checkmark          & \ding{55}             & 32.77 ± 3.2327          & 0.9291 ± 0.0274          \\
                                & \ding{55}             & \checkmark          & 32.92 ± 3.3730          & \textbf{0.9312 ± 0.0288}          \\
                                & \checkmark          & \checkmark          & \textbf{33.09 ± 3.2174} & 0.9305 ± 0.0276 \\ \hline
\end{tabular}
\end{table}

Note that w/o (I) and w/ (II) (the 3rd and 7th rows) do not explicitly ground parameters including $\{ \delta_k \}_{k=1}^K$, $\{\eta_k\}$, $\mu$ and leave them learnable; hence they use ground truth high-quality images as the intermediate supervision signal, which is similar to~\cite{chen2025comprehensive}. The experimental results show that though this approach can help improve performance, our FLAT achieves higher metrics, indicating the insufficient effectiveness of using the final target to supervise all cascades.

% \subsubsection{Impact of Formulation of $\mathcal{L}_\textrm{velocity}$} 
\subsubsection{\texorpdfstring{Impact of Formulation of $\mathcal{L}_\textrm{velocity}$}{Impact of Formulation of L velocity}}
\label{impact_of_formulation_of_L_velocity}

\begin{table}
\centering
\caption{Impact of formulation of velocity. The last row (N/A) indicates E2E-VarNet without FLAT
%network without FLAT.
, where `N/A' indicates `not available'. 
%We observe that there is not statistically significant difference between difference choices of formulation of $\mathcal{L}_\text{velocity}$ and the value of $w_\text{velocity}$, and that all these choices (all rows except the last row) are statistically significantly better than the base network (the last row), indicating that our FLAT is robust to the choice of formulation of $\mathcal{L}_\text{velocity}$, as long as the loss value is neither too large nor too small. No matter what formulation of $\mathcal{L}_\text{velocity}$ we use and what $w_\text{velocity}$ value we set, we can have better performance than the base model.
%\saum{please bold the statistically significant best performance. Otherwise people will think there is no difference compared to NA}
}
\label{tab:impact_of_form_of_velocity}
\scriptsize
\setlength{\tabcolsep}{6pt}
\begin{tabular}{cc|cc}
\hline
$\mathcal{L}_\text{velocity}$ formulation & $w_\text{velocity}$ & PSNR $\uparrow$                    & SSIM $\uparrow$                     \\ \hline
\multirow{3}{*}{$L_1$}                       & 1e-4                & 33.62 ± 3.3752 & 0.9412 ± 0.0269 \\
                                             & 1e-3                & 33.45 ± 3.2647          & 0.9387 ± 0.0269          \\
                                             & 1e-2                & 33.43 ± 3.3117          & 0.9375 ± 0.0277          \\ \hline
\multirow{4}{*}{$L_2$}                       & 1e-4                & 33.51 ± 3.1396          & 0.9349 ± 0.0254          \\
                                             & 1e-6                & 33.57 ± 3.2440          & 0.9357 ± 0.0269          \\
                                             & 1e-8                & 33.26 ± 3.1367          & 0.9343 ± 0.0269          \\
                                             & 1e-9                & 33.34 ± 3.3818          & 0.9382 ± 0.0278          \\ \hline
N/A                                          & 0                   & 32.77 ± 3.5402          & 0.9320 ± 0.0314          \\ \hline
\end{tabular}
\end{table}

%Though we use $L_1$ distance to compute the velocity loss term $\mathcal{L}_\text{velocity}$, we compare the formulation of $\mathcal{L}_\text{velocity}$ in~\cref{tab:impact_of_form_of_velocity}. We compare different choices of $w_\text{velocity}$ values. As we can observe in this table, there is no statistically significant difference between difference choices of $\mathcal{L}_\text{velocity}$, indicating that our $\mathcal{L}_\text{velocity}$ is agnostic to be computed using either $L_1$ or $L_2$. There is also no significant difference between difference choice of $w_\text{velocity}$, suggesting the robustness of this hyperparameter. We observe that even though we try different formulations of $\mathcal{L}_\text{velocity}$ and different values of $w_\text{velocity}$, all these combinations achieve significantly better performance than the base model, proving the robustness of our FLAT.
Although we use $L_1$ distance to compute the velocity loss term $\mathcal{L}_\text{velocity}$ in our main experiments, we compare alternative formulations in~\cref{tab:impact_of_form_of_velocity}, along with different values of $w_\text{velocity}$. As shown in the table, there is no statistically significant difference between $L_1$ and $L_2$ formulations, indicating that $\mathcal{L}_\text{velocity}$ is agnostic to the choice of distance metric. Similarly, performance remains stable across different values of $w_\text{velocity}$, suggesting robustness to this hyperparameter. Importantly, all combinations achieve statistically significant better performance than the base model (denoted by N/A).
%, demonstrating the overall robustness of FLAT.

\subsubsection{Impact of MRI Undersampling Level} 
\label{impact_of_mri_undersampling_level}
\cref{tab:impact_of_mri_acceleration_level} examines the impact of our approach on MRI undersampling levels. We observe that FLAT improves reconstruction performance at both $4\times$ and $8\times$ MRI acceleration, demonstrating its effectiveness across varying undersampling rates. We also observe that as the acceleration factor (acc) decreases, i.e., the quality of the undersampled image increases, the performance gain is slight, as the higher-quality input leaves less room for improvement.
%the improvement in reconstruction performance becomes slighter, because the image quality is already very high, limiting the potential for further improvement.

\begin{figure}
    \centering
    \includegraphics[width=1.\columnwidth]{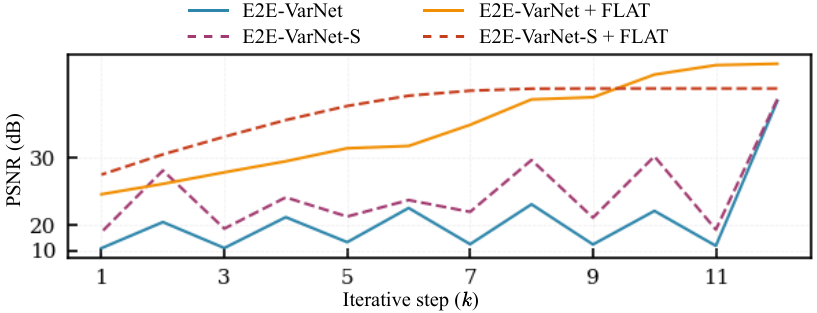}
    \caption{PSNR curves for 12-step iterations. Since intermediate cascades of unrolled networks suffer from image degradation, the PSNR curve of E2E-VarNet and E2E-VarNet-S shows a `zig-zag' pattern. Our FLAT, on the other hand, leads to a stable, increasing PSNR curve and achieves a better final PSNR value.
    % \saum{Compared to Fig 4, this curve is surprising. In Fig 4, even-k value (2,4,6...) outputs look better but your PSNR curves show that odd-k values (1,3,5...) are better}
    % \kh{Because the k-values are not consistent.}
    }
    \label{fig:psnr_curve_for_12_steps} 
\end{figure}

\subsubsection{Cascade Stability} 
\label{cascade_stability}
We analyze the average reconstruction quality across iterative cascades on the test set in~\cref{fig:psnr_curve_for_12_steps}. The iterations are from step 1 (first cascade output) to step 12 (final output). The unrolled baseline achieves a final PSNR $>$30dB, but exhibits unstable, fluctuating behavior across iterations, indicating degrading quality at intermediate steps. In contrast, FLAT shows smooth, monotonic improvements, with PSNR increasing steadily throughout the cascades.
%achieves an increasing PSNR curve throughout the iteration and shows smooth, monotonic improvements.

\begin{table}
\centering
\caption{Impact of MRI undersampling level. The \textbf{bold} metrics indicate that it has statistical significant difference. 
%E2E-VarNet and E2E-VarNet-S benefit from FLAT on both $4\times$ and $8\times$ undersampling of MRI reconstruction, proving the robustness of our FLAT on MRI undersampling levels. \saum{Notice for 4x, flat has lower ssim than e2e-varnet-s}
}
\label{tab:impact_of_mri_acceleration_level}
\scriptsize
\setlength{\tabcolsep}{8pt}
\begin{tabular}{cccc}
\hline
Acc Level                & Method                                     & PSNR $\uparrow$                    & SSIM $\uparrow$                     \\ \hline
\multirow{4}{*}{4} & \multicolumn{1}{c|}{E2E-VarNet}            & 41.73 ± 3.5283          & 0.9861 ± 0.0084          \\
                   & \multicolumn{1}{c|}{E2E-VarNet + FLAT}     & \textbf{42.27 ± 3.5074} &0.9866 ± 0.0080    \\ \cline{2-4} 
                   & \multicolumn{1}{c|}{E2E-VarNet-S}        & 40.83 ± 3.6202          & 0.9824 ± 0.0097    \\
                   & \multicolumn{1}{c|}{E2E-VarNet-S + FLAT} & \textbf{41.15 ± 3.3139} & 0.9821 ± 0.0084          \\ \hline
\multirow{4}{*}{8} & \multicolumn{1}{c|}{E2E-VarNet}            & 32.68 ± 3.3535          & 0.9324 ± 0.0291          \\
                   & \multicolumn{1}{c|}{E2E-VarNet + FLAT}     & \textbf{33.62 ± 3.3752} & \textbf{0.9412 ± 0.0269} \\ \cline{2-4} 
                   & \multicolumn{1}{c|}{E2E-VarNet-S}        & 32.64 ± 3.2147          & 0.9283 ± 0.0279          \\
                   & \multicolumn{1}{c|}{E2E-VarNet-S + FLAT} & \textbf{33.09 ± 3.2174} & \textbf{0.9305 ± 0.0276} \\ \hline
\end{tabular}
\end{table}

% \subsubsection{Impact of $w_\textrm{velocity}$, $\alpha$, $K$ and $\sigma$}
\subsubsection{\texorpdfstring{Impact of $w_\textrm{velocity}$, $\alpha$, $K$ and $\sigma$}{Impact of w verlocity, alpha, and sigma}}
\label{impact_of_alpha_and_k_and_sigma}
\cref{tab:impact_of_alpha_and_k_and_sigma} examines the impact of velocity supervision strength $w_\text{velocity}$, timestep-density factor $\alpha$, number of steps $K$ and the normalized scale factor $\sigma$ in the energy function. The first row ($w_\text{velocity}=10^{-4}$, $\alpha=4$, $K=12$, $\sigma=1$) indicates the default configuration, and is used in our main experiment. As shown in~\cref{tab:impact_of_alpha_and_k_and_sigma}, a too large $w_\text{velocity}$ value reduces the reconstruction performance. Such values will cause the training to focus too much on $\mathcal{L}_\text{velocity}$, and induce a slight ignorance of optimizing SSIM. Increasing $K$ provides more iterations for refinement, leading to improved reconstruction quality. Similarly, larger $\alpha$ values yield better image quality by concentrating more timesteps near $t=1$, where fine-scale refinement occurs. This distribution is crucial because our analysis reveals that most reconstruction steps focus on refinement rather than denoising: only a limited number of initial steps are needed to produce visually acceptable images. Therefore, a larger $\alpha$ allocates more computational resources to the refinement phase, while smaller $\alpha$ values under-utilize refinement steps, resulting in degraded image quality. For $\sigma$, we claim that in energy-based models, it is a normalization factor that works together with $\lambda(t)$ and $\delta_t$. We observe that the reconstruction metrics are insensitive to the selection of the $\sigma$ value, which is similar to what was found in~\cite{ho2020denoising}. We also observe that even if we choose the best combination, other combinations also achieve better results than the base network.

\begin{table}
\centering
\caption[Impact of hyperparameters alpha, K, and sigma in FLAT. The first row indicates the default set of these parameters, and the last row indicates the E2E-VarNet without FLAT, where `N/A' indicates `not available'. The rest rows are split into 4 groups, each group compares one parameter with the first row. As expected, our FLAT is robust to the choice of w velocity, insensitive to sigma, and a larger alpha and K lead to better performance. Though we only bold the best combination of these hyperparameters, all choices are statistically significantly better than the base network (the last row), which indicates the robustness of our FLAT.]{Impact of hyperparameters $\alpha$, $K$, and $\sigma$ in FLAT. The first row indicates the default set of these parameters, and the last row indicates the E2E-VarNet without FLAT, where `N/A' indicates `not available'. The rest rows are split into 4 groups, each group compares one parameter with the first row. As expected, our FLAT is robust to the choice of $w_\text{velocity}$, insensitive to $\sigma$, and a larger $\alpha$ and $K$ lead to better performance. Though we only \textbf{bold} the best combination of these hyperparameters, all choices are statistically significantly better than the base network (the last row), which indicates the robustness of our FLAT.
}
\label{tab:impact_of_alpha_and_k_and_sigma}
\scriptsize
\begin{tabular}{cccc|cc}
\hline
$w_\text{velocity}$ & $\alpha$ & $K$ & $\sigma$ & PSNR $\uparrow$                    & SSIM $\uparrow$                     \\ \hline
1e-4                & 4        & 12  & 1        & \textbf{33.62 ± 3.3752} & \textbf{0.9412 ± 0.0269} \\ \hline
1e-3                & 4        & 12  & 1        & 33.45 ± 3.2647          & 0.9387 ± 0.0269          \\
1e-2                & 4        & 12  & 1        & 33.43 ± 3.3117          & 0.9375 ± 0.0277          \\ \hline
1e-4                & 0        & 12  & 1        & 33.03 ± 3.4320          & 0.9352 ± 0.0293          \\
1e-4                & -0.5     & 12  & 1        & 33.04 ± 3.4889          & 0.9359 ± 0.0295          \\ \hline
1e-4                & 4        & 8   & 1        & 33.10 ± 3.4307          & 0.9374 ± 0.0281          \\
1e-4                & 4        & 10  & 1        & 33.25 ± 3.4533          & 0.9378 ± 0.0290          \\ \hline
1e-4                & 4        & 12  & 0.25     & 33.54 ± 3.4417          & 0.9394 ± 0.0283          \\
1e-4                & 4        & 12  & 0.5      & 33.38 ± 3.3189          & 0.9380 ± 0.0278          \\
1e-4                & 4        & 12  & 2        & 33.37 ± 3.3012          & 0.9375 ± 0.0274          \\ \hline
0                   & N/A      & 12  & N/A      & 32.77 ± 3.5402          & 0.9320 ± 0.0314          \\ \hline
\end{tabular}
\end{table}

% \subsubsection{2.5-D Multi-Slice Input Single-Slice Output Reconstruction}
% \label{impact_of_multi_slice_brainweb}
% \cref{tab:impact_of_multi_slice_brainweb} probes the impact of FLAT on a 2.5-D configuration, where multiple adjacent slices are stacked together as inputs for richer contexts to reconstruct single intermediate slice. For each indexed slice, we stack $N=3$ adjacent slices and feed the stacked slices to the regularizer CNN $f_\theta$. The data-consistency operator $A^{\top}(A\,\cdot\,-\,y)$ acts only on the center slice. We observe that FLAT improves reconstruction performance, suggesting its effectiveness under this 2.5-D setting.

% \begin{table}[ht]
% \centering
% \caption{2.5-D multi-slice input single-slice output reconstruction performance. 3 adjacent slices are stacked together to reconstruct the intermediate single slice. Metrics are reported on the center-slice reconstruction. The \textbf{bold} metrics indicate that it has statistically significant difference.}
% \label{tab:impact_of_multi_slice_brainweb}
% \scriptsize
% \begin{tabular}{c|cc}
% \hline
% Method                       & PSNR $\uparrow$  & SSIM $\uparrow$  \\ \hline
% E2E-VarNet                   & 33.09 ± 3.6273             & 0.9386 ± 0.0296             \\
% E2E-VarNet + FLAT (Ours)     & \textbf{33.66 ± 3.3361}             & \textbf{0.9421 ± 0.0268}             \\ \hline
% \end{tabular}
% \end{table}

\begin{table}
\centering
\caption{Metrics of reconstructing multiple slices. The \textbf{bold} metrics indicate that it has statistically significant differences. Our FLAT improves both PSNR and SSIM in this setting.}
\label{tab:impact_of_joint_multi_slice_brainweb}
\scriptsize
\begin{tabular}{c|cc}
\hline
Method                       & PSNR $\uparrow$  & SSIM $\uparrow$  \\ \hline
E2E-VarNet                   & 33.43 ± 3.1681             & 0.9416 ± 0.0259             \\
E2E-VarNet + FLAT (Ours)     & \textbf{34.01 ± 3.1384}             & \textbf{0.9453 ± 0.0239}             \\ \hline
\end{tabular}
\end{table}

\subsubsection{Impact of FLAT on Reconstructing Multiple Slices}
\label{impact_of_joint_multi_slice_brainweb}
We evaluate reconstructing multiple slices in~\cref{tab:impact_of_joint_multi_slice_brainweb}. The cascade iterate is applied on a stack of $N=5$ adjacent slices, and the data-consistency step is applied on all $N$ slices. The FLAT loss term is computed and summed over all $N$ slices during training. In~\cref{tab:impact_of_joint_multi_slice_brainweb}, there is a performance improvement after applying FLAT to simultaneously reconstructing multiple slices, demonstrating its effectiveness on this setting.

\begin{table}
\centering
\caption{Metrics of evaluation on random sampling MRI. Our FLAT achieves better PSNR and SSIM on the random sampling setting. The \textbf{bold} metrics indicate that it has statistically significant differences.}
\label{tab:impact_on_random_sampling_brainweb}
\scriptsize
\begin{tabular}{c|cc}
\hline
Method                       & PSNR $\uparrow$  & SSIM $\uparrow$  \\ \hline
E2E-VarNet                   & 32.99 ± 3.35             & 0.9990 ± 0.0296             \\
E2E-VarNet + FLAT (Ours)     & \textbf{33.82 ± 3.4270}             & \textbf{0.9402 ± 0.0287}             \\ \hline
\end{tabular}
\end{table}

\subsubsection{Impact of FLAT on Random MRI Sampling}
\label{impact_on_random_sampling_brainweb}
We employ 1-D random sampling on 2-D slices with $8\times$ acceleration and center fraction 8\% for MRI undersampling. The evaluation results are shown in~\cref{tab:impact_on_random_sampling_brainweb}. Our FLAT improves MRI reconstruction metrics not only in equispaced sampling, but also in random sampling.

\begin{table}
\centering
\vspace{-.18in}
\caption{Impact of FLAT on fastMRI knee multi-coil dataset. Similar to fastMRI knee single-coil dataset, our FLAT achieves the best PSNR and comparable SSIM. The \textbf{bold} metrics indicate that it has statistically significant difference. Though we only \textbf{bold} the last row, both of our FLAT are significantly better than the baselines, indicating the effectiveness of our method.}
\label{tab:experiment_on_fastmri_knee_multi_coil_dataset}
\scriptsize
\begin{tabular}{c|cc}
\hline
Method                       & PSNR $\uparrow$                    & SSIM $\uparrow$                  \\ \hline
E2E-VarNet                   & 27.70 ± 4.7316          & 0.8046 ± 0.1033 \\
E2E-VarNet-S               & 27.24 ± 4.1562          & 0.7937 ± 0.1069       \\ \hline
E2E-VarNet + FLAT (Ours)     & 29.03 ± 4.4319          & 0.7986 ± 0.1133       \\
E2E-VarNet-S + FLAT (Ours) & \textbf{29.15 ± 4.1153} & 0.8058 ± 0.1055 \\ \hline
\end{tabular}
\end{table}

\subsubsection{Impact of FLAT on FastMRI Knee Multi-coil Dataset}
\label{impact_of_flat_on_fastmri_multi_coil}
We explore the impact of our FLAT on the fastMRI knee multi-coil dataset. This dataset consists of 973 volumes in the training set, 118 in the validation set, and 199 in the testing set. Each scan was acquired from a 3T or 1.5T clinical system. The number of slices in each volume ranges from 28 to 50. The in-plane resolution is $0.5\text{mm}\times 0.5\text{mm}$, and the slice thickness is $3\text{mm}$. The number of coils is 15. All slices are cropped to size $320\times 320$. The number of 2-D slices for training, validation, and testing are 34742, 4092, and 7135, respectively. We show the experimental results in~\cref{tab:experiment_on_fastmri_knee_multi_coil_dataset}. Similar to the performance on fastMRI knee single-coil dataset, our FLAT achieves better performance in PSNR, which proves the positive impact of FLAT on this dataset.
\begin{table}
\centering
\caption{Impact of FLAT on ADMM-CSNet. The \textbf{bold} metrics indicate that it has statistically significant difference. The metrics are improved after applying our FLAT, indicating the effectiveness.}
\label{tab:impact_of_flat_on_admm_net}
\scriptsize
\begin{tabular}{c|cc}
\hline
Method                       & PSNR $\uparrow$  & SSIM $\uparrow$  \\ \hline
ADMM-CSNet                   & 26.65 ± 3.4181             & 0.7984 ± 0.0794             \\
ADMM-CSNet + FLAT (Ours)     & \textbf{26.92 ± 3.3674}             & \textbf{0.8076 ± 0.0765}             \\ \hline
\end{tabular}
\end{table}

\subsubsection{Impact of FLAT on ADMM-CSNet}
\label{sec:impact_of_flat_on_admm_csnet}
To study the impact of FLAT on variable splitting-based MRI reconstruction methods, we report the impact of FLAT on ADMM-CSNet~\cite{yang2018admm} in ~\cref{tab:impact_of_flat_on_admm_net}. We observe that our FLAT is effective in improving the performance of ADMM-CSNet, indicating that our FLAT benefits variable-splitting methods in MRI reconstruction.

\begin{table}
\centering
\caption{Impact of unrolled-network-style loss $\mathcal{L}_\text{reconstruction}$ on Flow Matching. The \textbf{bold} metrics indicate that it has statistically significant difference. The $\mathcal{L}_\text{velocity}$ is improved by $\mathcal{L}_\text{reconstruction}$ on Flow Matching, indicating a mutual benefit of these two loss terms.}
\label{tab:impact_of_l_reconstruction_on_flow_matching}
\scriptsize
\begin{tabular}{c|cc}
\hline
loss                                                            & PSNR                    & SSIM                     \\ \hline
$\mathcal{L}_\text{velocity}$                                   & 27.56 ± 3.3404          & 0.8037 ± 0.0786          \\
$\mathcal{L}_\text{velocity}+\mathcal{L}_\text{reconstruction}$ & \textbf{30.02 ± 3.1856} & \textbf{0.8839 ± 0.0317} \\ \hline
\end{tabular}
\end{table}

% \subsubsection{Benefit of $\mathcal{L}_\text{reconstruction}$ on Flow Matching}
\subsubsection{\texorpdfstring{Benefit of $\mathcal{L}_\text{reconstruction}$ on Flow Matching}{Benefit of L reconstruction on Flow Matching}}
\label{flow_matching}
We conduct an exploration on the benefit of the training paradigm of unrolled networks on the Flow Matching. We use SMSflow~\cite{zhang2024mutli} as our base Flow Matching method. \cref{tab:impact_of_l_reconstruction_on_flow_matching} shows that the performance of SMSflow is limited on MRI reconstruction, implying the reason for the limited utility of this method in such tasks. To enable an unrolled-network-style loss, at each iteration of training, we first sample the final output $\hat{x}_1$ using the network, and then compute the SSIM loss term as $\mathcal{L}_\text{reconstruction}$ on this predicted $\hat{x}_1$. We examine the impact of this unrolled-network-style loss term $\mathcal{L}_\text{reconstruction}$ on Flow Matching in~\cref{tab:impact_of_l_reconstruction_on_flow_matching}. This loss term benefits the Flow Matching, leading to a 2.46 dB increase in PSNR and 0.0802 increase in SSIM.

\section{Discussion}
\label{discussion}

\textbf{Keeping Information in Intermediate Steps. }Explicit timestep control in~\eqref{discretized_ode} helps keep more information in intermediate steps, and thereby improves the accuracy of prediction in the following cascade. 
% As shown in~\figref{fig:intermediate}, without explicitly controlling the timesteps, the intermediate steps of unrolled networks drop a lot of information, such as background or low-frequency information. 
In~\figref{fig:intermediate}, for E2E-VarNet, half of the intermediate steps drop a lot of information, such as background or low-frequency information, leading to low-quality images. E2E-VarNet-S maintains better image quality than E2E-VarNet, implying its robustness in intermediate steps, but still lacks sufficient information for visualization. This suggests unstable cascade iteration in~\cref{fig:psnr_curve_for_12_steps}, where the PSNR values among cascades change dramatically for E2E-VarNet and E2E-VarNet-S. This issue hurts the final reconstruction performance. The output of the previous cascade is the input of the next cascade, and the previous cascade drops some information, so the next cascade takes imperfect data as input. This will undoubtedly harm the reconstruction of the next cascade. On the contrary, as our FLAT keeps the necessary low-frequency information in the intermediate steps, the reconstruction achieves better performance.

\textbf{Performance Gain. } The \textit{t}-test on~\cref{tab:comparison} illustrates the significant PSNR gain of our FLAT on the Brainweb and MRBrainS13 datasets. On the fastMRI knee single coil dataset, though the performance gain is not significant, our approach still achieves the best numerical PSNR. \cref{tab:impact_of_alpha_and_k_and_sigma} and~\cref{tab:impact_of_form_of_velocity} illustrate the robustness of our FLAT, where no matter what hyperparameters are used and what form of $\mathcal{L}_\text{velocity}$ is employed, our FLAT helps improve the metric numerically.

\section{Conclusion}
In this work, we introduce Flow Matching as a training paradigm to address the unstable trajectory of unrolled networks. We theoretically show that unrolled networks are discrete implementations of approximate conditional probability flows, establishing an algebraic connection between unrolled optimization and continuous-time generative dynamics. This connection reveals that effective unrolled training requires (i) parameters to be grounded in Flow Matching, and (ii) intermediate supervision. Building on this, we introduce FLow-Aligned Training (FLAT), which explicitly initializes parameters and aligns velocity at intermediate steps, analogous to Flow Matching. Across three MRI datasets, FLAT produces high-quality reconstructions and markedly more stable intermediate behavior than vanilla unrolled baselines.

\appendices

\section*{Appendix}

\subsection{Dataset Description}
\label{appendix:dataset_description}

\textbf{Brainweb}~\cite{cocosco1997brainweb}. This is a publicly available MR brain image simulation tool, which provides clear, structured MR images. We synthesized 20 T1-weighted brain MR image volumes with a voxel resolution of 1 mm. Each volume consists of 362 slices. All slices are cropped to $256\times 256$. A 10/5/5 train/val/test data split was used for this dataset. The number of 2-D slices for training, validation, and testing are 3620, 1810, and 1810, respectively.

\textbf{MRBrainS13}~\cite{mendrik2015mrbrains}. This dataset consists of 20 MR imaging cases. We only use T1-weighted MR image volumes in our experiments. Each volume has a voxel size $0.96 \mathrm{mm} \times 0.96 \mathrm{mm} \times 3 \mathrm{mm}$, and contains 48 slices. All slices are cropped to $224\times 224$. We split this dataset into train, val, and test sets, respectively containing 5, 7, and 8 volumes. The number of 2-D slices for training, validation, and testing are 240, 336, and 384, respectively.

\textbf{fastMRI Knee}~\cite{zbontar2018fastmri}. We use the single-coil data from this dataset. To obtain ground truth data in \textit{k}-space, we only use the training set. Inside this set, there are 973 volumes in total. The number of slices in each volume ranges from 28 to 50. The in-plane resolution is $0.5\textrm{mm}\times 0.5\textrm{mm}$, and the slice thickness is $3\textrm{mm}$. 486, 195, and 292 volumes are used for training, validation, and testing, respectively. All slices are cropped to size $320\times 320$. The number of 2-D slices for training, validation, and testing are 17287, 6945, and 10510, respectively.

\subsection{Implementation Details}
\label{appendix:implementation_details}
Though our FLAT is backbone-agnostic, we use the widely-used MRI reconstruction network E2E-VarNet~\cite{sriram2020end} as our base model. Consistent with our image-domain formulation, the cascade iterate $x^{(k)}$ is a single-channel complex image, and we implement $v_\theta(x_{t_k}, t_k;y)$ as
\begin{align*}
    & v_\theta(x_{t_k}, t_k;y)
    = A^{\top} \left( A\, f_\theta(x^{(k)}) - y \right)
\label{eq:impl}
\end{align*}
where $A=M\mathcal{F}S$ is the image-to-undersampled-k-space forward operator and $A^{\top}=S^{*}\mathcal{F}^{-1}M^{\top}$ is its adjoint, which performs inverse mask, inverse Fourier transform, and sensitivity-weighted coil combination back to a single-channel image. $f_\theta$ is a CNN regularizer operating in the image domain. The sensitivity maps inside $S$ are estimated by a separate network as in vanilla E2E-VarNet for multi-coil data; for single-coil data, $S=I$. We try two sub-versions of this network: one does not share weights across cascades, following the vanilla configuration of E2E-VarNet; and the other shares weights, which follows a set of unrolled networks such as ISTA-Net, and is marked as `E2E-VarNet-S' in tables and figures. FLAT is conducted on both of these networks. We use the same training and loss configurations for E2E-VarNet and FLAT. Following~\cite{sriram2020end}, we use SSIM loss as $L_\text{reconstruction}$.
% We present more implementation details in~\cref{implementation_details}.

% Though our FLAT is backbone-agnostic, we use the widely-used MRI reconstruction network E2E-VarNet~\cite{sriram2020end} as our base model. We try two sub-versions of this network: one does not share weights across cascades, following the vanilla configuration of E2E-VarNet; and the other shares weights, which follows a set of unrolled networks such as ISTA-Net, and is marked as `E2E-VarNet-S' in tables and figures. FLAT is conducted on both of these networks. We use the same training and loss configurations for E2E-VarNet and FLAT. Following~\cite{sriram2020end}, we use SSIM loss as $L_\text{reconstruction}$. We present more implementation details, including the implementation of the base model, hyperparameters, and evaluation metrics, in~\cref{implementation_details}.

% \kh{No space left.}
\textbf{Sampling of Timesteps.} 
To supervise our $K$-step unrolled network against the continuous trajectory, we select $K+1$ discrete points. We employ a time schedule $\{t_k \}_{k=0}^{K}$ to manage these time points. Similar to~\cite{karras2022elucidating}, we use a non-linear time sequence, which is denser near $t=1$. This sequence can be either uniform or non-uniform; empirically, a non-uniform schedule denser near $t=1$ yields better performance. This provides us with a set of ideal supervisory targets $\{x^*_{t_k} \}_{k=0}^{K}$ sampled along the target flow. Specifically, we sample $\{t_k \}_{k=0}^{K}$ as:
\begin{equation}
    t_k = 1-\left(1 - k/K\right)^{(1 + \alpha)}
    \label{eq:time_sequence}
\end{equation}
where $\alpha$ is a hyperparameter controlling the density of $\{t_k \}_{k=0}^{K}$. The ideal targets $\{x_{t_k}^* \}_{k=0}^K$ are then computed using linear interpolation following~\eqref{eq:ideal_trajectory}.

\textbf{Hyperparameters.} For simplicity, we fix $\lambda(t_k)=1$ and set $\sigma=1$ in \eqref{discretized_ode}. We explore different choices of $\sigma$ in~\cref{tab:impact_of_alpha_and_k_and_sigma}. To balance loss terms to the same numeric scale, we set $w_{\text{velocity}}=10^{-4}$. We train our network from scratch with AdamW optimizer, using a learning rate of $10^{-3}$ and a batch size of 1. We train for 200 epochs on a single NVIDIA A6000 GPU.

% \section{More Results on ADMM-CSNet}

\subsection{Limitations} 
\label{limitations}
% \textbf{Limited to First-Order Gradient-Based Unrolled Networks.} The unrolled networks are based on various iterative algorithms, such as gradient descent~\cite{zhang2018ista, sriram2020end} and ADMM~\cite{sun2016deep, yang2018admm}. However, the theoretical foundation of FLAT is limited to the first-order gradient-based unrolled networks. Expanding Flow Matching to other sets of unrolled networks to resolve their inherent issues is non-trivial and challenging. We will explore this domain as our future work.

% \textbf{Approximate Theoretical Correspondence.}The energy $E(x;y)$ in~\eqref{energy_function} is the exact negative log-posterior only when $x$ is drawn from the target distribution $x_1$. Applying the same energy functional at intermediate trajectory points $x_t$ is a heuristic extension: the prior $p(x_t)$ at intermediate states differs from the prior $p(x_1)$ over fully-sampled images, and the data fidelity term $\|Ax_t-y\|^2$ carries a different probabilistic interpretation along the flow. The correspondence established in~\eqref{eq:correspondence} is therefore an algebraic connection between unrolled network updates and a discretized energy-based flow, rather than a strict probabilistic equivalence for all $t\in[0,1]$. The practical implication—that Flow Matching principles can guide parameter setting and intermediate supervision in unrolled networks—is validated empirically in~\cref{sec:experiments}.

\textbf{Approximate Theoretical Correspondence.} The energy-based conditional probability flow in~\eqref{energy_function} is strictly valid as a Bayesian posterior only for the target distribution $x_1$. Applying this energy to intermediate trajectory points $x_t$ is a heuristic approximation: the prior $p(x_t)$ at intermediate states is not the prior over fully-sampled images, and the data fidelity term carries a different probabilistic meaning along the flow. Consequently, the correspondence in~\eqref{eq:correspondence} is an algebraic connection, not a strict probabilistic equivalence for all $t\in[0,1]$. Establishing a fully rigorous conditional probability flow formulation for the entire trajectory is left to future work.

% \textbf{Complex-Valued Data Coverage.} Among the three datasets, only fastMRI provides natively complex raw k-space measurements; Brainweb and MRBrainS13 only release real-valued (magnitude) phantoms or scans, so we form their complex inputs by zero-padding the imaginary channel. This is a common but limited surrogate for true complex acquisitions, and the strongest validation of FLAT on realistic complex data therefore comes from the fastMRI single-coil and multi-coil experiments.

\textbf{Marginal Performance Gain on FastMRI Dataset.} Though our FLAT achieves significant PSNR improvement and best numerical SSIM on Brainweb, MRBrainS13 and the fastMRI dataset, it does not achieve the highest SSIM metric on the fastMRI dataset. The performance drop is impacted by the `trade-off' between our $\mathcal{L}_\text{velocity}$ and vanilla SSIM loss. The SSIM loss term helps optimize the SSIM value directly, but our $\mathcal{L}_\text{velocity}$ is not designed to directly optimize this metric on the complex data distribution, so the PSNR increases at the cost of a slight SSIM drop. We argue that this trade-off is meaningful if we focus more on PSNR.

% \section*{Acknowledgment}
% This work was partially supported by NSF Grant CCF-2144901, NIH Grants R01NS143143, R01CA297843 and R21EB040046. The content is solely the responsibility of the authors and does not necessarily represent the official views of the NIH. 
% \section{References}
% \bibliography{main}

% Need to add author bio here.

\end{document}